\documentclass[journal]{IEEEtran}
\usepackage{pepe_macros}
\begin{document}

\title{Compression and Interpretability of Deep Neural Networks via Tucker Tensor Layer: From First Principles to Tensor Valued Back-Propagation}

\author{Giuseppe G. Calvi, \textit{Student Member, IEEE}, Ahmad Moniri, Mahmoud Mahfouz, Qibin Zhao, \textit{Senior Member, IEEE},  Danilo P. Mandic, \textit{Fellow, IEEE}}

\markboth{Journal of \LaTeX\ Class Files,~Vol.~6, No.~1, January~2007}%
{Shell \MakeLowercase{\textit{et al.}}: Bare Demo of IEEEtran.cls for Journals}

\maketitle

\begin{abstract}

This work aims to help resolve the two main stumbling blocks in the application of Deep Neural Networks (DNNs), that is, the exceedingly large number of trainable parameters and their physical interpretability. This is achieved through a tensor valued approach, based on the proposed Tucker Tensor Layer (TTL), as  an alternative to the dense weight-matrices of DNNs. This allows us to treat the weight-matrices of general DNNs as a matrix unfolding of a higher order weight-tensor. By virtue of the compression properties of tensor decompositions,  this enables us to introduce a novel and efficient framework for exploiting the multi-way nature of the weight-tensor in order to dramatically reduce the number of DNN parameters. We also derive the tensor valued back-propagation algorithm within the TTL framework, by extending the notion of matrix derivatives to tensors. In this way, the physical interpretability of the Tucker decomposition is exploited to gain physical insights into the NN training, through the process of computing gradients with respect to each factor matrix. The proposed framework is validated on both synthetic data, and the benchmark datasets MNIST, Fashion-MNIST, and CIFAR-10. Overall, through the ability to provide the relative importance of each data feature in training, the TTL back-propagation is shown to help  mitigate the ``black-box" nature inherent to NNs. Experiments also illustrate that the TTL achieves a 66.63-fold compression on MNIST and Fashion-MNIST, while, by simplifying the VGG-16 network, it achieves a 10\% speed up in training time, at a comparable performance.


\end{abstract}


\IEEEpeerreviewmaketitle

\section{Introduction}

Deep Neural Networks (DNNs) are the state-of-the-art machine-learning methodology which has delivered improved performance in numerous areas involving large-scale data, such as computer vision, speech recognition, and time-series analysis \cite{Lecun2015}. Within DNNs, Recurrent Neural Networks (RNNs) are among the most successful machine learning approaches for sequence modelling \cite{Graves2013, Mandic2001}, while Convolutional Neural Networks (CNNs) are particularly well-suited for image classification \cite{Krizhevsky2012, Simonyan2014}. However, despite the involvement of specialised high performance hardware, the application of DNNs has also highlighted the necessity to reduce the typically very long processing times, as DNN models often comprise thousands of nodes and millions of learning parameters \cite{Russakovsky2015}.

One way to tackle this issue is by employing a tensor approach to reduce the sheer scale of the DNN learning problem through the super-compression ability inherent to tensors. Indeed, owing to their compression properties, tensor decompositions (TDs) are just emerging as means  to optimise NNs, by approximating an original ``raw" network with its low-rank tensor representation which exhibits several orders of magnitude fewer parameters, while minimally affecting its overall performance.

Tensors represent a generalization of matrices and vectors, which benefit from the power of multilinear algebra to flexibly and efficiently account for multi-way relationships in data while at the same time preserving the structural information \cite{Mandic2015, Kolda2009}. Similarly to their matrix counterpart, latent factors in tensors can be extracted via TDs; the most commonly used are the Canonical Polyadic \cite{Bro1997}, the Tensor Train \cite{Oseledets2011}, and the Tucker \cite{Tucker1963} decompositions (CPD, TT, and TKD, respectively). In \cite{Novikov2015}, the TT format was employed to efficiently represent the dense weight matrix of the fully connected layers NNs, requiring only a few parameters. The ability of this approach to super-compress the NN weight matrix at a little sacrifice in accuracy was demonstrated on standard datasets. The work in  \cite{Kossaifi2017} proposed the Tensor Contraction Layer (TCL) and Tensor Regression Layer (TRL) approaches, to reduce input dimensionality while preserving its multilinear tensor structure, and expressing the outputs as a mapping between two tensors. This work was extended in \cite{Cao2017} with low-rank constraints imposed on the TRL, while the work in \cite{Lebedev2014} tackled the question of improving training time by merging CNNs and the CPD.

Further research in the area, in particular involving the combination of the TKD with NNs, includes the work in \cite{Kim2015} and \cite{Zhong2019}. In \cite{Kim2015} the TKD was used as a convolution kernel for mobile applications, while the authors in \cite{Zhong2019} adopt a more theoretical approach and propose a method to adaptively adjust the dimensions of the weight tensor in a layer-wise fashion. Both methods have been shown to achieve significant NN parameter compression and  consequent reduction in training time.

Despite these advances, there is still a void in the literature when it comes to directly deriving the back-propagation algorithm for DNNs in terms of tensor latent factors within TDs, while at the same time providing physical meaning to the model. The authors of the seminal paper \cite{Novikov2015} are the only to have analytically re-derived back-propagation \cite{Rumelhart1986} in terms of the tensor factors stemming from the TT decomposition, however, the dimensionality and order of the employed TTs were arbitrarily chosen, so that, despite the achieved compression, the results were not physically interpretable. Most other research in the field tends to rely on automatic differentiation for back-propagation \cite{Kossaifi2017, Cao2017, Kim2015, Zhong2019}. Although this is convenient for implementation purposes, as automatic differentiation is quite efficient, there is still a necessity for a deep analytical understanding of how the  errors are propagated; this is fundamental to the explainability of DNNs.  In this work, we set out to address this void by introducing the Tucker Tensor Layer (TTL), which replaces the dense weight-matrix of a fully connected layer within DNNs with TKD factors. The TKD is chosen by virtue of its ability to provide low-rank representations of tensors while at the same time preserving the latent information within their structure \cite{DeLathauwer2000_2}. Within our approach, we treat the weight-matrix as the unfolding (flattening) of a higher order weight-tensor, and proceed to fully derive back-propagation in terms of the underlying tensor factors, that is, directly in the employed tensor format.

This work therefore aims to introduce a novel analytical framework for the optimization of NNs through compression of their fully-connected layers, together with increasing the interpretability of DNN models by exploiting the desirable properties of TKD. For rigour, we address the problem starting from the very core aspects, basing our analysis on the fundamentals of vector and matrix derivatives, and extending these to tensors, while providing rigorous mathematical arguments for their manipulation. In this way, we demonstrate that this novel and theoretically well-founded derivation is a prerequisite for obtaining physical insights into the training process. For example, we show that the ability of the proposed approach to keep track of the evolution of the gradients of each TKD factor during training makes it possible to reveal valuable information regarding the significance of input features within the DNN model. This, in turn, promises to be of high practical importance as: (i) new physically meaningful insights are provided into how NNs perform classification, thus mitigating their inherent well-known ``black-box" nature, and (ii) inherent means to perform computationally cheap data augmentation follow naturally. The proposed approach is verified through application case studies on synthetic data and the benchmark MNIST \cite{Lecun1998}, Fashion-MNIST \cite{Xiao2017}, and CIFAR-10 \cite{Krizhevsky2009} datasets. In addition, owing to the compressive properties of our proposed framework, the TTL achieves a 66.63 fold compression for NNs trained on MNIST and Fashion-MNIST. It also allows to simplify the VGG-16 CNN \cite{Simonyan2014} applied to CIFAR-10, resulting in a 10\% training time speed up, at a little accuracy trade-off thanks to data augmentation based on gradient information. Comparable performance to the uncompressed NNs is achieved in all cases. 

The contributions of this work can be summarized as follows:
\begin{enumerate}
	\item The proposed Tucker Tensor Layer (TTL) is introduced and a full analytic derivation of back-propagation directly in terms of the TKD factors is provided;
	\item We demonstrate that the TTL can achieve significant compression of DNN parameters while maintaining comparable performance to uncompressed NNs;
	\item By leveraging on the developed theoretical framework, novel insights into the training process are provided via gradient inspection, hence offering a new interpretability to results;
	\item The information from point (3) is employed to perform cheap data augmentation, ultimately leading to the simplification and training speed-up of more complex NNs.
\end{enumerate}

\section{Notation}

\begin{table}[H]
	\centering
	\caption{Main matrix and tensor nomenclature.}
	\label{table:nomenclature}

	\begin{tabular}{ll}
		\hline
		
		\vspace{-2mm}	& \\
		
		$\tn{X} \in \mathbb{R}^{I_1 \times I_2 \times \cdots \times I_N}$ & \begin{tabular}[c]{@{}l@{}} Tensor of order $N$ of\\ size $I_1 \times I_2 \times \cdots \times I_N$\end{tabular} \vspace{2mm} \\

		$x_{i_1i_2\cdots i_N}=\tn{X}(i_1,i_2,\cdots,i_N)$ 	& $(i_1,i_2,\cdots,i_N)$ entry of $\tn{X}$ \vspace{2mm} \\
		$x$, $\mathbf{x}$, $\mathbf{X} $	& Scalar, vector, matrix  \vspace{2mm}\\

		$\mathcal{X}_{(n)} \in \mathbb{R}^{I_n \times I_1 I_2 \dots I_{n-1}I_{n+1} \dots I_N   } $ & \begin{tabular}[c]{@{}l@{}}Mode-$n$ unfolding of  \\  tensor $\tn{X}$\end{tabular} \vspace{2mm} \\		
		
		$\mathbf{A}^{(n)}$ & \begin{tabular}[c]{@{}l@{}}Factor matrices, \\  in tensor decompositions\end{tabular} \vspace{2mm} \\
		
		$(\cdot)^T$, $(\cdot)^{-1}$ & \begin{tabular}[c]{@{}l@{}} Transpose and inverse operators \\  for matrices\end{tabular} \vspace{2mm} \\

		$\circ$, $\otimes, \odot$	& \begin{tabular}[c]{@{}l@{}} Outer, Kronecker,  \\  and Hadamard products\end{tabular} \vspace{2mm}  \\

		$\text{vec}(\tn{X}) = \mathbf{x} \in \mathbb{R}^{I_1 I_2 \cdots I_N}$ & \begin{tabular}[c]{@{}l@{}} Vectorization of \\   tensor $\tn{X}$ \end{tabular} \vspace{2mm} \\

		$||\cdot||_F$	& Frobenius norm  \vspace{2mm}  \\

		$\mathbf{I}_M$	&  \begin{tabular}[c]{@{}l@{}}Identity matrix of\\ size $M\times M$\end{tabular} \vspace{2mm}  \\

		$\mathbf{1}_M$ & \begin{tabular}[c]{@{}l@{}}Vector of ones of\\ size $M$\end{tabular} \vspace{2mm}  \\
		\hline
	\end{tabular}
\end{table}

TABLE \ref{table:nomenclature} summarizes the main tensor nomenclature used throughout this work. In particular, the mode-$n$ unfolding of a tensor $\tn{X} \in \mathbb{R}^{I_1 \times I_2 \times \dots \times I_N}$ represents a rearrangement of the indices of $\tn{X}$, resulting in a matrix $\mathcal{X}_{(n)}  \in \mathbb{R}^{I_n \times I_1 I_2 \dots I_{n-1}I_{n+1} \dots I_N  } $ with entries
\begin{equation}
x_{i_1i_2\cdots i_N} = (\mathcal{X}_{(n)})_{i_n,\overline{i_1\dots i_{n-1} i_{n+1} \dots i_N   }}
\end{equation}
as per the Little-Endian convention \cite{Dolgov2014}. An outer product, denoted by $\circ$ \cite{Mandic2015},  of tensors $\tn{A} \in \mathbb{R}^{I_1 \times I_2 \times \cdots \times I_N}$ and $\tn{B} \in \mathbb{R}^{J_1 \times J_2 \times \cdots \times J_M}$ yields a tensor $\tn{C} \in \mathbb{R}^{I_1 \times I_2 \times \cdots \times I_N \times J_1 \times J_2 \times \cdots \times J_M}$ with entries
\begin{equation}
c_{i_1 i_2 \dots i_N j_1 j_2 \dots j_M} = a_{i_1 i_2 \dots i_N} b_{j_1 j_2 \dots j_M}
\end{equation}
whereas the $(m, n)$-contraction, denoted by $\times^m_n$, between tensors $\tn{X} \in \mathbb{R}^{I_1 \times \cdots \times I_n \times \cdots \times I_N}$ and $\tn{Y} \in \mathbb{R}^{J_1 \times \cdots \times J_m \times \cdots \times J_M    }$ yields a tensor $\tn{Z} \in \mathbb{R}^{ I_1 \times \dots \times I_{n-1} \times I_{n+1} \times \dots \times I_N \times J_1 \times \dots \times J_{m-1} \times J_{m+1} \times \dots \times J_M      }$, with entries
\begin{equation}\label{eq:cont}
\begin{aligned}
&z_{i_1,\dots,i_{n-1}, i_{n+1}, \dots, i_N, j_1, \dots, j_{m-1}, j_{m+1}, \dots, j_M   } =\\
& = \sum_{i_n}^{I_n} x_{i_1, \dots, i_{n-1}, i_n, i_{n+1}, \dots, i_N y_{j_1, \dots, j_{m-1}, i_n, j_{m+1}, \dots, j_M}   }
\end{aligned}
\end{equation}
In particular, we employ the operator $\times^2_n$, which, by convention, is equivalently expressed as $\times_n$.

\section{Theoretical Background}

\subsection{Tucker Decomposition}

The Tucker decomposition (TKD) was introduced by Ledyard Tucker in the 1960s the use in psychometrics and chemometrics \cite{Tucker1963}. It is analogous to a higher form of Principal Components Analysis (PCA) \cite{Wold1987,Tucker1963}, since it decomposes an $N$-th order tensor $\tn{X}$ into $N$ factor matrices, which are projected onto a core tensor which describes the relationship among all matrix entries. For example, for a $3$-rd order tensor, the TKD takes the form of
\begin{equation}
\begin{aligned}
\tn{X} &= \sum_{r_1}^{R_1}\sum_{r_2}^{R_2}\sum_{r_3}^{R_2} g_{r_1 r_2 r_3} \mathbf{u}_{r_1} \circ \mathbf{u}_{r_2} \circ \mathbf{u}_{r_3}\\
& = \tn{G} \times_1 \mathbf{U}^{(1)} \times_2 \mathbf{U}^{(2)} \times_3 \mathbf{U}^{(3)}
\end{aligned}
\end{equation}
where $\tn{G} \in \mathbb{R}^{R_1 \times R_2 \times R_3}$ is the core tensor with multilinear rank $\{ R_1, R_2, R_3  \}$, and $\mathbf{U}^{(1)} \in \mathbb{R}^{I_1 \times R_1}$, $\mathbf{U}^{(2)} \in \mathbb{R}^{I_2 \times R_2}$, $\mathbf{U}^{(3)} \in \mathbb{R}^{I_3 \times R_3}$. An illustration of the TKD for $3$-rd order tensors is portrayed in Fig. \ref{fig:tkd3}.

\begin{figure}[b]
	\centering
	\includegraphics[width=0.9\linewidth]{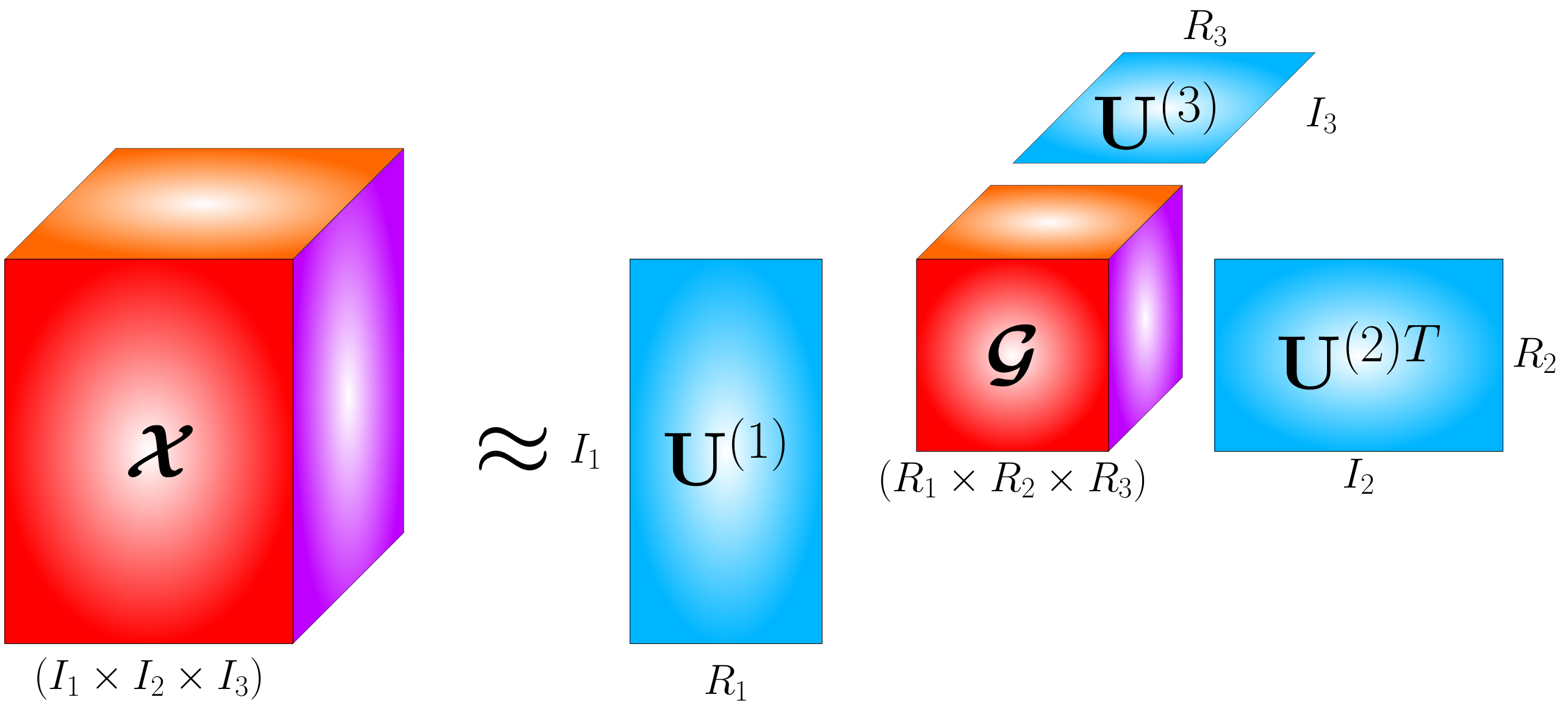}
	\caption{Tucker Decomposition (TKD) for a $3$-rd order tensor $\tn{X} \in \mathbb{R}^{I_1\times I_2 \times I_3}$. The original raw data tensor, $\tn{X}$, is decomposed into a core tensor $\tn{G} \in \mathbb{R}^{R_1 \times R_2 \times R_3}$  and factor matrices $\mathbf{U}^{(1)} \in \mathbb{R}^{I_1 \times R_1}$,  $\mathbf{U}^{(2)} \in \mathbb{R}^{I_2 \times R_2}$,  $\mathbf{U}^{(3)} \in \mathbb{R}^{I_3 \times R_3}$ }
	\label{fig:tkd3}
\end{figure}

Similarly, a TKD for an $N$-th order tensor $\tn{X} \in \mathbb{R}^{I_1 \times I_2 \times \cdots \times I_N}$  takes the form 
\begin{equation}
\begin{aligned}
\tn{X} = & \tn{G} \times_1 \mathbf{U}^{(1)} \times_2 \mathbf{U}^{(2)} \times_3 \cdots \times_N \mathbf{U}^{(N)}
\end{aligned}
\end{equation}
where $\tn{G} \in \mathbb{R}^{R_1 \times R_2 \times \cdots \times R_{N+1}}$ and $\mathbf{U}^{(n)} \in \mathbb{R}^{I_n \times R_n}$. 
Throughout this work, we make significant use of the mode-$n$ unfolding of a tensor within the TKD format, given by  
\begin{equation}\label{eq:nthunfold}
\begin{aligned}
\mathcal{X}_{(n)}= \mathbf{U}^{(n)}\mathcal{G}_{(n)}\big(\mathbf{U}^{(N)}  \otimes& \cdots \otimes \mathbf{U}^{(n-1)} \otimes \\ 
&\otimes\mathbf{U}^{(n+1)} \otimes \cdots  \otimes \mathbf{U}^{(1)}   \big)^T
\end{aligned}
\end{equation}

\subsection{Properties of the Kronecker Product}

If $\mathbf{A} \in \mathbb{R}^{M \times N}$ and $\mathbf{B} \in \mathbb{R}^{P \times Q}$, then the operation $\mathbf{A} \otimes \mathbf{B}$ yields a matrix $\mathbf{C} \in \mathbb{R}^{MP \times NQ}$, with elements $c_{P(r-1)+v, Q(s-1)+w} = a_{rs}b_{vw}$. The Kronecker product is a bilinear and associative operator, which is non-commutative and permutation equivalent. 
Most important properties for our analysis are  the transpose property of the Kronecker product, 
\begin{equation}
(\mathbf{A} \otimes \mathbf{B})^T = \mathbf{A}^T \otimes \mathbf{B}^T
\end{equation}
and the following identity
\begin{equation}
\text{vec}(\mathbf{AXB}) = (\mathbf{B}^T \otimes \mathbf{A}) \text{vec}(\mathbf{X})
\end{equation}

\section{Tensor Derivatives}
To map the parameters of DNNs into the proposed tensor framework, it is a prerequisite to extend the notion of vector and matrix derivatives to higher-order tensors. To this end, we shall first reformulate Definition 4 and Definition 6 from \cite{Magnus1985} in a way which suits our analysis, as follows.
\begin{definition}
	A \textit{ball}  in $\mathbb{R}^{I_1 \times I_2 \times \cdots \times I_N}$ of radius $r$ and centre $\tn{C}$ is denoted by
	\begin{equation}
	B(\tn{C}; r) = \{ \tn{X}: \tn{X} \in \mathbb{R}^{I_1 \times I_2 \times \cdots \times I_N}, || \tn{X} - \tn{C} ||_F < r    \}
	\end{equation}
	where $\tn{C} \in \mathbb{R}^{I_1 \times I_2 \times \cdots \times I_N}$ is an interior point of a set $S$ in $\mathbb{R}^{I_1 \times I_2 \times \cdots \times I_N}$.
\end{definition}

\begin{definition}\label{def:setballten}
	Let $S$ be a set in $\mathbb{R}^{I_1 \times I_2 \times \cdots \times I_N}$, and let $F: S \mapsto \mathbb{R}^{J_1 \times J_2 \times \cdots \times J_M}$ be a tensor-valued function operating on $S$. Notice that, in general, $N \neq M$, that is, $F$ does not necessarily map a tensor to a tensor of the same order. Let the tensor $\tn{C} \in \mathbb{R}^{I_1 \times I_2 \times \cdots \times I_N}$ be an interior point of $S$, and let $B(\tn{C}; r) \subset S$ be a ball with centre $\tn{C}$ and radius $r$. Let $\tn{E}$ be an arbitrary tensor in $\mathbb{R}^{I_1 \times I_2 \times \cdots \times I_N}$, with $||\tn{E}||_F<r$, so that $(\tn{C}+\tn{E}) \in B(\tn{C}; r)$. If there exists a matrix $\mathbf{A} \in \mathbb{R}^{J_1J_2\cdots J_M \times I_1I_2 \cdots I_N}$, which depends on $\tn{C}$ but not on $\tn{E}$, so that
	\begin{equation}
	\text{vec}(F(\tn{C} + \tn{E})) = \text{vec}(F(\tn{C})) + \mathbf{A}(\tn{C})\text{vec}(\tn{E}) + \text{vec}\big(R_\tn{C}(\tn{E})\big)
	\end{equation}
	for all $\tn{E} \in \mathbb{R}^{I_1 \times I_2 \times \cdots \times I_N}$ with $||\tn{E}||_F < r$, where the remainder, $R_{\tn{C}} (\tn{E}$), is of a smaller order than $||\tn{E}||_F$ as $\tn{E} \rightarrow 0$, that is
	\begin{equation}
	\lim_{\tn{E} \rightarrow \mathbf{0}} \frac{R_{\tn{C}}(\tn{E})}{||\tn{E}||_F} = \mathbf{0} 
	\end{equation}
	then the function $F$ is said to be \textit{differentiable} at $\tn{C}$. Define
	\begin{equation}\label{eq:id}
	  \mathbf{A}(\tn{C}) \text{vec}(\tn{E}) = \text{vec}\big(   dF(\tn{C}; \tn{E})    \big) \in \mathbb{R}^{J_1 J_2 \cdots J_M}
	\end{equation}
	then the tensor $dF(\tn{C}; \tn{E}) \in \mathbb{R}^{J_1 \times J_2 \times \cdots \times J_M}$ is called the \textit{first differential of $F$ at} $\tn{C}$ \textit{with increment} $\tn{E}$.
\end{definition}

In view of Definition \ref{def:setballten}, the properties of vector calculus can be readily extended to tensors, because, instead of considering the tensor function $F: S \mapsto \mathbb{R}^{J_1 \times J_2 \times \cdots \times J_M}$, we may consider the vector function $f: \text{vec}(S) \mapsto \mathbb{R}^{J_1 J_2 \cdots J_M \times 1}$ defined by
\begin{equation}
f(\text{vec}(\tn{X})) = \text{vec}(F(\tn{X}))
\end{equation}
Then it follows that
\begin{equation}
\text{vec}(dF(\tn{C}; \tn{E})) = df(\text{vec}(\tn{C}); \text{vec}(\tn{E}))
\end{equation}
This justifies the following definition. 
\begin{definition}\label{def:jacten} 
	Consider a tensor $\tn{X} \in \mathbb{R}^{I_1 \times I_2 \times \cdots \times I_N}$ and the differentiable tensor function $F: \tn{X} \mapsto \mathbb{R}^{J_1 \times J_2 \times \cdots \times J_M}$. Then, the matrix in $\mathbb{R}^{J_1J_2 \cdots J_M \times I_1 I_2 \cdots I_N}$, given by
	\begin{equation}
	\frac{\partial F(\tn{X}) }{\partial \tn{X} } = D(F(\tn{X})) = Df(\text{vec}(\tn{X})) = \frac{\partial f ( \text{vec}(\tn{X})) }{\partial \text{vec}(\tn{X})  }
	\end{equation}
	is referred to as the \textit{derivative} of $F$ at $\tn{X}$.
\end{definition}

A special case of Definition \ref{def:jacten} occurs when the input to a tensor function is a matrix, as formalized in the following definition.  
\begin{definition}\label{def:jacten2}
	Consider a matrix, $\mathbf{X} \in \mathbb{R}^{M \times P}$, and the differentiable tensor function, $F: \mathbf{X} \mapsto \mathbb{R}^{I_1 \times I_2 \times \cdots \times I_N}$. Then the matrix in $\mathbb{R}^{I_1I_2 \cdots I_N \times MP}$, given by
	\begin{equation}\label{eq:equiv}
	\frac{\partial F(\mathbf{X}) }{\partial \mathbf{X} } = D(F(\mathbf{X})) = Df(\text{vec}(\mathbf{X})) = \frac{\partial f ( \text{vec}(\mathbf{X})) }{\partial \text{vec}(\mathbf{X})  }
	\end{equation}
	is referred to as the \textit{derivative} or \textit{Jacobian} of $F$ at $\mathbf{X}$.
\end{definition}

\begin{remark}\label{rem:id}
	From Definition \ref{def:setballten} and the ``identification theorems" in \cite{Magnus1985}, it follows that the conditions on the first differential in (\ref{eq:id})  hold if and only if $ \mathbf{A}(\tn{C}) = D(F(\tn{C}))$.
\end{remark}

With the above analysis, we can now establish a connection between tensor differentiability and the permutation of tensor entries.
\begin{theorem}\label{th:perm}
	Consider a matrix $\mathbf{X} \in \mathbb{R}^{M \times P}$ and the differentiable tensor function $F: \mathbf{X} \mapsto \mathbb{R}^{I_1 \times I_2 \times \cdots \times I_N}$. Consider the tensor mapping $\tn{Y} = F(\mathbf{X})$, and denote by $\mathcal{Y}_{(n)}$ the mode-$n$ unfolding of $\tn{Y}$. Then $\forall n \in \{ 1, 2, \cdots, N\}$, $D(\mathcal{Y}_{(n)})$ is a permuted version of $D(\tn{Y}) = D(F(\tn{X}))$, i.e. $D(\tn{Y})$ and $D(\mathcal{Y}_{(n)})$ are matrices containing the same elements, but arranged differently. We can write this as
	\begin{equation}
	\vspace{-2mm}
	D(\mathcal{Y}_{(n)}) = \mathcal{P}(D(\tn{Y}))
	\end{equation}
	where $\mathcal{P}$ is a permutation operator. 
\end{theorem}

\begin{proof}
	Let $\mathbf{X} \in \mathbb{R}^{M\times P}$ and $\tn{Y} = F(\mathbf{X}) \in \mathbb{R}^{I_1 \times I_2 \times \cdots \times I_N}$.  From Definition \ref{def:jacten2} and (\ref{eq:equiv}) we have  
	\begin{equation}
	D(\tn{Y}) = \frac{\partial \tn{Y}}{\partial \mathbf{X}} = \frac{\partial \text{vec}(\tn{Y}) } {\partial \text{vec}(\mathbf{X}) } \in \mathbb{R}^{I_1I_2 \cdots I_N \times MP}
	\end{equation}
	At the same time, this results in  $\mathcal{Y}_{(n)} \in \mathbb{R}^{I_n \times I_1 I_2 \cdots I_{n-1}I_{n+1} \cdots I_N}$. Next,  from Definition \ref{def:jacten2} and (\ref{eq:equiv}) we obtain
	\begin{equation}
	\begin{aligned}
	D(\mathcal{Y}_{(n)}) = \frac{\partial   \mathcal{Y}_{(n)}  }{\partial  \mathbf{X}    } =& \frac{\partial \text{vec}(\mathcal{Y}_{(n)})  }{\partial \text{vec} ( \mathbf{X} )} \\ &\in  \mathbb{R}^{I_n I_1 \cdots I_{n-1} I_{n+1} \cdots I_N \times MP}
	\end{aligned}
	\end{equation}
	Hence, $D(\tn{Y})$ and $D(\mathcal{Y}_{(n)})$ have the same dimensionality. Because $\text{vec}(\tn{Y})$ and $\text{vec}(\mathcal{Y}_{(n)})$ clearly have the same elements but arranged differently, $D(\mathcal{Y}_{(n)})$ is a permuted version of $D(\tn{Y})$.
\end{proof}

\begin{corollary}\label{cor:nm}
	Consider a matrix $\mathbf{X} \in \mathbb{R}^{M \times P}$ and the differentiable tensor function $F: \mathbf{X} \mapsto \mathbb{R}^{I_1 \times I_2 \times \cdots \times I_N}$. Let $\tn{Y} = F(\tn{X})$, and denote by $\mathcal{Y}_{(n)}$ and $\mathcal{Y}_{(m)}$ respectively the mode-$n$ and mode-$m$ unfoldings of $\tn{Y}$. Then $D(\mathcal{Y}_{(n)})$ and  $D(\mathcal{Y}_{(m)})$ are of the same dimensionality, but with permuted elements. This can be formalized as
	\begin{equation}
	D(\mathcal{Y}_{(n)}) = \mathcal{P} (D ( \mathcal{Y}_{(m)}  )   )
	\end{equation}
	where $\mathcal{P}$ is a permutation operator. 
\end{corollary}

\begin{proof}
	Follows immediately from Theorem \ref{th:perm}.
\end{proof}

\section{Tucker Tensor Layer (TTL)}\label{sec:tkdnn}
For simplicity, we shall initially omit nonlinearities in the NN model, however, for rigour, they are considered afterwards.
\subsection{Representation}
In general, a fully connected layer within a NN can be expressed as \vspace{-2mm}
\begin{equation}\label{eq:ten2}
\mathbf{y} = \mathbf{W}\mathbf{x} + \mathbf{b}
\vspace{-1mm}
\end{equation}
where $y \in \mathbb{R}^{M}$ is the output, $\mathbf{x} \in \mathbb{R}^N$ the input, and $\mathbf{W} \in \mathbb{R}^{M \times N}$ is the connecting matrix of weights. Now, consider an input tensor of order $N$, $\tn{X} \in \mathbb{R}^{I_1 \times I_2 \times \cdots \times I_N}$, and a weight tensor $\tn{W} \in \mathbb{R}^{I_1 \times I_2 \times \cdots \times I_{N+1}}$, of order $N+1$. Our tensor-valued derivation of back-propagation will be based on the investigation of the function $F: \tn{X} \mapsto \mathbb{R}^{I_{N+1}} $, where $I_{N+1}$ is the number of classes, defined by
\begin{equation}\label{eq:ten}
F(\tn{X}) = \mathcal{W}_{(N+1)}\text{vec}(\tn{X}) + \mathbf{b}
\end{equation}
where $\mathbf{b} \in \mathbb{R}^{I_{N+1}}$ is a bias vector, $\mathcal{W}_{(N+1)} \in \mathbb{R}^{I_{N+1} \times I_1 I_2 \cdots I_N}$ is the $(N+1)$-mode unfolding of tensor $\tn{W}$, and $\text{vec}(\tn{X}) \in \mathbb{R}^{I_1 I_2 \cdots I_N}$ is the vectorization of tensor $\tn{X}$. In (\ref{eq:ten}), we have set $\mathbf{y} = F(\tn{X})$, $\mathbf{W} = \mathcal{W}_{(N+1)}$, and $\mathbf{x} = \text{vec}(\tn{X})$. As shown in the following, the representation in (\ref{eq:ten}) has the advantage of allowing for a compressed version of the NN layer, via tensor decompositions.

A tensor $\tn{W}$ can be represented in the TKD format as
\begin{equation}\label{eq:tuck}
\begin{aligned}
\tn{W} &= \tn{G} \times_1 \mathbf{U}^{(1)} \times_2 \mathbf{U}^{(2)} \times_3 \cdots \times_{N+1} \mathbf{U}^{(N+1)}
\end{aligned}
\end{equation}
where $\tn{G} \in \mathbb{R}^{R_1 \times R_2 \times \cdots \times R_{N+1}}$ is the core tensor, and $\mathbf{U}^{(n)} \in \mathbb{R}^{I_n \times R_n}$ the corresponding factor matrices. This implies that the $(N+1)$-mode matrix unfolding of $\tn{W}$, denoted by $\mathcal{W}_{(N+1)}$, can be expressed as
\begin{equation}\label{eq:w_unfolded}
\begin{aligned}
\mathcal{W}_{(N+1)} & =  \mathbf{U}^{(N+1)}\mathcal{G}_{(N+1)} \big (   \mathbf{U}^{(N)} \otimes \cdots \otimes \mathbf{U}^{(1)}    \big)^T \\
& = \mathbf{U}^{(N+1)}\mathcal{G}_{(N+1)} \Motimes_{i=N}^{1} \mathbf{U}^{(i)T}
\end{aligned}
\end{equation}
Upon substituting (\ref{eq:w_unfolded}) into (\ref{eq:ten}) we have
\begin{equation}\label{eq:F}
F(\tn{X}) = \bigg [ \mathbf{U}^{(N+1)}\mathcal{G}_{(N+1)} \Motimes_{i=N}^{1} \mathbf{U}^{(i)T} \bigg] \text{vec}(\tn{X}) + \mathbf{b}
\end{equation}
\begin{remark}
	The compression of the DNN parameters through the TTL approach is achieved by selecting the size of the modes of the core tensor $\tn{G}$ within the TKD to be smaller than those of the weight tensor $\tn{W}$, that is $R_n < I_n$, for $n=1, 2, \dots, N+1$.
\end{remark}

\subsection{Learning via Tensor-Valued Back-Propagation}\label{sec:back}

Neural networks are generally trained with stochastic gradient descent algorithms, where at each step the gradient is computed using the back-propagation procedure \cite{Rumelhart1986}. Back-propagation starts by computing the gradient of a loss function $L$ w.r.t. the NN output, then, given the gradient $\frac{\partial L}{\partial F(\tn{X})}$, proceeds sequentially through the layers of the NN, but in a reversed order. When applied to the fully connected layer as expressed in (\ref{eq:ten}) back-propagation computes \cite{Novikov2015}
\begin{equation}\label{eq:grad}
\begin{aligned}
&\frac{\partial L}{\partial \text{vec}{(\tn{X})}  } =  \mathcal{W}_{(N+1)}^T \frac{\partial L}{\partial F(\tn{X})  }\\[3pt]
&\frac{\partial L}{\partial \mathbf{b}} =\frac{\partial L}{\partial F(\tn{X})}\\[3pt]
&\frac{\partial L}{ \partial \mathcal{W}_{(N+1)}} =    \frac{\partial L}{\partial F(\tn{X})} \text{vec}(\tn{X})^T
\end{aligned}
\end{equation} 
Since $L$ is a scalar, the gradients in (\ref{eq:grad}) are of dimensionality $\frac{\partial L}{\partial \text{vec}{(\tn{X})}  } \in \mathbb{R}^{I_1 I_2 \cdots I_N}$, $\frac{\partial L}{\partial \mathbf{b}} \in \mathbb{R}^{I_{N+1}}$, and $\frac{\partial L}{ \partial \mathcal{W}_{(N+1)}} \in \mathbb{R}^{I_{N+1} \times I_1 I_2 \cdots I_N}$. This allows, at each iteration of the algorithm, for an update of the weight matrix as $\mathcal{W}_{(N+1)_{t+1}}  = \mathcal{W}_{(N+1)_t} + \mu \frac{\partial L}{ \partial \mathcal{W}_{(N+1)_t}}$ where $\mu$ is a step size. 

To derive the tensor-valued back-propagation based on the Tucker model in (\ref{eq:tuck}), we first note that the gradient $\frac{\partial L  }{\partial \mathbf{U}^{(n)} }$ is a function of $\frac{\partial  L }   {    \partial  F(\tn{X})      } $,      $\frac{\partial F(\tn{X})  }{\partial  \mathcal{W}_{(N+1)}  }$, $\frac{\partial \mathcal{W}_{(N+1)}  }{\partial  \mathbf{U}^{(n)} }$, to give the derivative
\begin{equation}\label{eq:part}
\frac{\partial L  }{\partial \mathbf{U}^{(n)} }   =   \Omega \bigg( \frac{\partial  L }{\partial  F(\tn{X}) }, \frac{\partial F(\tn{X})  }{\partial  \mathcal{W}_{(N+1)}  },    \frac{\partial \mathcal{W}_{(N+1)}  }{\partial  \mathbf{U}^{(n)} }    \bigg)
\end{equation}
Then, we can define the following shorthand notation for the dimensionalities involved: $\underline{R}_{-n}  = R_1 R_2 \cdots R_{n-1}R_{n+1} \cdots R_{N+1}$, $\underline{I}_{-n}  = I_1 I_2 \cdots I_{n-1}I_{n+1} \cdots I_{N+1}$, $\underline{I}_{N}  = I_1 I_2 \cdots I_{N}$ and $\underline{I}_{N+1} = I_1 I_2 \cdots I_{N+1}$.
From the rules of the Kronecker product, we now obtain
\begin{equation}
\frac{\partial  F(\tn{X}) }{\partial \mathcal{W}_{(N+1)}   } = \text{vec}(\tn{X})^T \otimes \mathbf{I}_{I_{N+1}} \in \mathbb{R}^{I_{N+1} \times \underline{I}_{N+1}}
\end{equation}
A closer inspection of equation  (\ref{eq:part}) shows that. Regarding the term  $\frac{\partial   \mathcal{W}_{(N+1)}  }{\partial  \mathbf{U}^{(n)} } \in \mathbb{R}^{\underline{I}_{N+1}  \times R_n I_n }$, by Corollary \ref{cor:nm}, this matrix can be found as a permutation of  $\frac{\partial \mathcal{W}_{(n)}  }{\partial  \mathbf{U}^{(n)} }$, which is a matrix of the same dimensionality and containing the same elements. In other words, 
\begin{equation}
\frac{\partial \mathcal{W}_{(N+1)}  }{\partial  \mathbf{U}^{(n)}} = \mathcal{P} \bigg(  \frac{\partial \mathcal{W}_{(n)}  }{\partial  \mathbf{U}^{(n)} }  \bigg) = \mathbf{P}_n \frac{\partial \mathcal{W}_{(n)}  }{\partial  \mathbf{U}^{(n)} }
\end{equation}
where $\mathbf{P}_n \in \mathbb{R}^{\underline{I}_{N+1} \times \underline{I}_{N+1}  }$ is a permutation matrix that satisfies
\begin{equation}\label{eq:Pn}
\text{vec}(\mathcal{W}_{(n)}) = \mathbf{P}_n \text{vec}(\mathcal{W}_{(N+1)})
\end{equation}
The task hence boils down to computing
\begin{equation}\label{eq:part2}
\frac{\partial   L  }{\partial   \mathbf{U}^{(n)}  } = \Omega \bigg(  \frac{\partial  L }{\partial  F(\tn{X}) }, \frac{\partial F(\tn{X})  }{\partial  \mathcal{W}_{(N+1)}  },  \mathbf{P}_n \frac{\partial  \mathcal{W}_{(n)}    }   {\partial \mathbf{U}_{(n)}   }    \bigg)
\end{equation}
The mode-$n$ unfolding of the weight tensor $\tn{W} \in \mathbb{R}^{I_1 \times I_2 \times \cdots \times I_{N+1} }$ is given by 
\begin{equation}
\begin{aligned}
\mathcal{W}_{(n)}& =\mathbf{U}^{(n)}\mathcal{G}_{(n)} \big(   \mathbf{U}^{(N+1)} \otimes  \cdots \otimes&&\\ &&\hspace*{-4cm}\otimes\mathbf{U}^{(n+1)} \otimes \mathbf{U}^{(n-1)} \otimes \cdots \otimes \mathbf{U}^{(1)}  \big)^T\\
&= \mathbf{U}^{(n)}\mathcal{G}_{(n)} \Motimes_{ \substack{i=N+1 \\ i \neq n}}^{1} \mathbf{U}^{(i)T} \in \mathbb{R}^{I_n \times \underline{I}_{-n}}
\end{aligned}
\end{equation}
while its partial derivative w.r.t. $\mathbf{U}^{(n)}$ now becomes
\begin{equation}\label{eq:dWndUn}
\begin{aligned}
\frac{\partial \mathcal{W}_{(n)}  }  {\partial \mathbf{U}^{(n) } } & = \big(   \mathbf{U}^{(N+1)} \otimes \cdots \otimes&&\\ 
&&\hspace*{-6cm}\otimes\mathbf{U}^{(n+1)} \otimes \mathbf{U}^{(n-1)} \otimes \cdots \otimes \mathbf{U}^{(1)}      \big) \mathcal{G}_{(n)}^T \otimes \mathbf{I}_{I_n}\\
& = \Motimes_{ \substack{i=N+1 \\ i \neq n}}^{1 } \mathbf{U}^{(i)}\mathcal{G}_{(n)}^T \otimes \mathbf{I}_{I_n} \in \mathbb{R}^{\underline{I}_{N+1}\times R_n I_n  }
\end{aligned}
\end{equation}
From (\ref{eq:part}) we require $\frac{\partial   L  }{\partial  \mathbf{U}^{(n)}  } \in \mathbb{R}^{I_n \times R_n}$, which is the matricization of $\frac{\partial   L  }{\partial \text{vec} (\mathbf{U}^{(n)})  } \in \mathbb{R}^{R_n  I_n}$. Therefore, to find $\frac{\partial   L  }{\partial  \mathbf{U}^{(n)}  } \in \mathbb{R}^{I_n \times R_n}$, it is sufficient to compute $\frac{\partial   L  }{\partial  \text{vec} (\mathbf{U}^{(n)} ) } \in \mathbb{R}^{R_n I_n}$ and reshape the result accordingly. Since $\frac{\partial  L  }{\partial F(\tn{X})	} \in \mathbb{R}^{I_{N+1}}$, $\frac{\partial F  (\tn{X})  }{\partial     \mathcal{W}_{(N+1)}  }  \in \mathbb{R}^{I_{N+1} \times \underline{I}_{N+1}}$, $\frac{\partial  \mathcal{W}{(n)}    }   {\partial \mathbf{U}_{(n)}   } \in \mathbb{R}^{\underline{I}_{N+1} \times R_n I_n  }$ and $\frac{\partial  L  }{\partial  \text{vec}(\mathbf{U}^{(n)})   } \in \mathbb{R}^{R_n I_n}$,  this results in
\begin{equation}\label{eq:dLdU}
\frac{\partial  L  }{\partial  \text{vec}(\mathbf{U}^{(n)})   } = \bigg( \mathbf{P}_n \frac{\partial  \mathcal{W}{(n)}    }   {\partial \mathbf{U}_{(n)}   } \bigg)^T  \frac{\partial  F  (\tn{X})  }{\partial     \mathcal{W}_{(N+1)}  }^T \frac{\partial  L  }{\partial F(\tn{X})}
\end{equation}
Similarly, since $\frac{\partial L  }{\partial  \mathcal{G}_{(N+1)}    } \in \mathbb{R}^{R_{N+1} \times \underline{R}_N}$ can be computed by considering $\frac{\partial L  }{\partial \text{vec} (\mathcal{G}_{(N+1)})    } \in \mathbb{R}^{\underline{R}_{N+1}}$, we arrive at
\begin{equation}\label{eq:dLdG}
\frac{\partial L  }{\partial \text{vec} (\mathcal{G}_{(N+1)})} = \frac{\partial \mathcal{W}_{(N+1)}  }{\partial  \mathcal{G}_{(N+1)  } }^T \frac{\partial   F(\tn{X})   }{\partial \mathcal{W}_{(N+1)}  }^T \frac{\partial  L }{\partial  F(\tn{X}) }
\end{equation}
where, 
\begin{equation}\label{eq:dWdGN1}
\frac{\partial \mathcal{W}_{(N+1)}   }{\partial \mathcal{G}_{(N+1)}   } = (\mathbf{U}^{(N)}  \otimes \cdots \otimes \mathbf{U}^{(1)}    ) \otimes \mathbf{U}^{(N+1)} \in \mathbb{R}^{\underline{I}_{N+1} \times \underline{R}_{N+1}    }   
\end{equation}

\subsection{Verifying the TTL Gradients}\label{sec:gradcheck}

Equations (\ref{eq:dLdU}) and (\ref{eq:dLdG}) can be verified based on Definition \ref{def:setballten}. In particular, defining the cost function as a general function $L: (F(\tn{X}), \mathbf{d}) \mapsto \mathbb{R}$, where $\mathbf{d}$ is a desired output and $F(\tn{X})$ is defined as in (\ref{eq:F}), we can verify the gradient $\frac{\partial L}{\partial \mathbf{U}^{(n)}}$ (i.e. a permuted version of $\frac{\partial L}{\partial  \text{vec}(\mathbf{U}^{(n)})   }$) by treating all inputs to $L$ as constants except for $\mathbf{U}^{(n)}$, and confirming that
\begin{equation}\label{eq:ratU}
\lim_{\mathbf{E} \rightarrow \mathbf{0}} \frac{R_{\mathbf{U}^{(n)}}(\mathbf{E})}{||\mathbf{E}||_F} =\mathbf{ 0} 
\end{equation} 
where, from Remark \ref{rem:id},
\begin{equation}
\begin{aligned}
R_{\mathbf{U}^{(n)}}(\mathbf{E})  &= L(\mathbf{U}^{(n)}  + \mathbf{E} ) - L(\mathbf{U}^{(n)} ) - \mathbf{A}(\mathbf{U}^{(n)}) \text{vec}(\mathbf{E})\\
 = L(\mathbf{U}^{(n)}&  + \mathbf{E} ) - L(\mathbf{U}^{(n)} ) - \bigg( \frac{\partial L}{\partial  \text{vec}(\mathbf{U}^{(n)})   }\bigg)^T \text{vec}(\mathbf{E})
\end{aligned}
\end{equation}

Similarly, to verify the gradient $\frac{\partial     L   }{\partial  \tn{G}     }$ (i.e. a permuted version of $\frac{\partial L    }{\partial  \text{vec}(\mathcal{G}_{(N+1)})   }$  ), we have to treat all inputs to $L$ as constants, except for $\tn{G}$, and confirm that
\begin{equation}\label{eq:ratG}
\lim_{\tn{E} \rightarrow \mathbf{0}} \frac{R_{\tn{G}}(\tn{E})}{||\tn{E}||_F} = \mathbf{0} 
\end{equation}
where
\begin{equation}\label{eq:remG}
\begin{aligned}
R_{\tn{G}}(\tn{E})  &= L(\tn{G}  + \tn{E} ) - L(\tn{G} ) - \mathbf{A}(\tn{G}) \text{vec}(\tn{E})\\
& = L(\tn{G}  + \tn{E} ) - L(\tn{G} ) - \bigg( \frac{\partial L}{\partial  \text{vec}(\tn{G})   }\bigg)^T \text{vec}(\tn{E})
\end{aligned}
\end{equation}

To demonstrate the correctness of this approach, we defined a $3$-rd order data tensor $\tn{X} \in \mathbb{R}^{5 \times 5 \times 5}$, the elements of which were drawn from a normal distribution, such that $\text{vec}(\tn{X}) \sim \mathcal{N}(0,1)$, and a $4$-th order weight tensor $\tn{W} \in \mathbb{R}^{5 \times 5 \times 5 \times 3}$, also with elements drawn from a normal distribution, i.e. $\text{vec}(\tn{W}) \sim \mathcal{N}(0,1)$. From $\tn{W}$, the core tensor $\tn{G}$ and the factor matrices $\mathbf{U}^{(n)}$, $n = 1, \dots, 4$ were obtained through a TKD with multi-linear rank $(5,5,5,3)$. The desired output $\mathbf{d} \in \mathbb{R}^3$ was arbitrarily assigned, and the mean-square error (MSE) was employed as a cost function, i.e. $L = \frac{1}{2}||F(\tn{X}) - \mathbf{d}  ||^2_F$, implying $\frac{\partial L}{\partial F(\tn{X})} = F(\tn{X})-\mathbf{d}$. The gradients were computed according to the procedure in Section \ref{sec:back}, and were verified according to Equations (\ref{eq:ratU})-(\ref{eq:remG}). The results are shown in Fig. \ref{fig:ratios}, and conform with our analysis.

\begin{figure*}[t]
	\centering
	\includegraphics[width=1\linewidth,  trim={0 0cm 0 0cm}, clip]{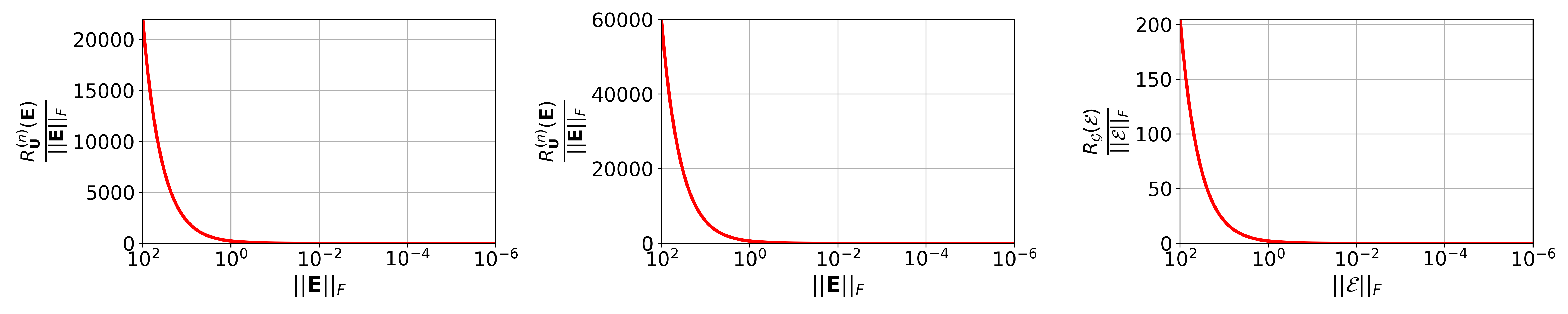}
	\caption{Verification of the gradients by confirming that the remainder tends to 0 faster than the magnitude of the perturbation, as per (\ref{eq:ratU}) and (\ref{eq:ratG}). \textit{Left}: Gradient verification for $\frac{\partial   L   }{\partial    \mathbf{U}^{(2)}    }$. \textit{Centre}: Gradient verification for $\frac{\partial   L   }{\partial    \mathbf{U}^{(4)}    }$. \textit{Right}: Gradient verification for $\frac{\partial   L   }{\partial    \tn{G}    }$. Notice the inverted abscissa and the logarithmic scale.}
	\label{fig:ratios}
\end{figure*}


\section{Algorithm Implementation}

To illustrate the algorithmic implementation of the results obtained in Section \ref{sec:tkdnn}, we next consider a dataset consisting of $N$-th order tensors $\tn{X} \in \mathbb{R}^{I_1 \times I_2 \times \cdots \times I_N}$ and their corresponding labels, that is $\{  \tn{X}_m, y_m   \}_m^{M}$, $m=1, \dots, M$, where $M$ is the size of the dataset. For an efficient representation, the data is organized in a matrix $\ten{X} \in \mathbb{R}^{\underline{I}_N \times M}$, defined as
\begin{equation}
\ten{X} = 
\begin{bmatrix}
\text{vec}(\tn{X}_1), & \text{vec}(\tn{X}_2), & \dots, & \text{vec}(\tn{X}_M)  
\end{bmatrix}
\end{equation}
so that Equation (\ref{eq:ten}) can be re-written as
\begin{equation}
F(\ten{X})   = \mathcal{W}_{(N+1)} \ten{X} + \mathbf{B} \equiv \mathbf{\underline{F}}
\end{equation}
where $\mathbf{B} \in \mathbb{R}^{I_{N+1} \times M}$ is a matrix of biases, and $F(\ten{X}) \in \mathbb{R}^{I_{N+1} \times M}$. Note that $\mathbf{B}^{(l)} \in \mathbb{R}^{I_{N+1}^{(l)} \times M  }$ must apply the same biases to each element in the dataset, and, as such, has the structure of $\mathbf{B}^{(l)} = \mathbf{b}^{(l)}\mathbf{1}^T_M$
\begin{equation}
\mathbf{B}^{(l)} = \mathbf{b}^{(l)}\mathbf{1}^T_M
\end{equation}
where $\mathbf{b}^{(l)} \in \mathbb{R}^{I_{N+1}^{(l)}}$ are the actual biases.

Consider now a neural network composed of layers indexed by $l = 0, 1, 2, \dots, L$, where $l = 0$ is the input layer, $l = L$ is the output layer, and $l = 1, \dots, L-1$ are the hidden layers. For generality, assume that all layers are TTLs. With a slight modification in our notation, for each layer we denote inputs and outputs as summarized in Table \ref{table:layers}.

\begin{table}[b]
	\centering
	\caption{Inputs and outputs of network layers.}
	\label{table:layers}
	\begin{tabular}{ll}
		\hline
		
		\vspace{-2mm}	& \\
		
		$\ten{Z}^{(0)} = \ten{X}$ & \begin{tabular}[c]{@{}l@{}}  Output of layer $l=0$:\\ the input to the \\network\end{tabular} \vspace{1mm} \\

		$\mathcal{W}_{(N+1)}^{(l)}$ 	&\begin{tabular}[c]{@{}l@{}}  Weight matrix for \\ layers $l= 1, \dots, L$\end{tabular}\vspace{1mm} \\
		
		$\mathbf{U}^{(n)}_{(l)}$, $\mathcal{G}_{(N+1)}^{(l)}$ 	&\begin{tabular}[c]{@{}l@{}}  Factor matrices  and \\ core tensor associated \\ to $\mathcal{W}_{(N+1)}^{(l)}$ for layers \\ $l= 1, \dots, L$   \end{tabular}\vspace{1mm} \\
		
		$\ten{F}^{(l)} = \mathcal{W}_{(N+1)}^{(l)} \ten{Z}^{(l-1)} + \mathbf{B}^{(l)}$	& \begin{tabular}[c]{@{}l@{}} Input to layers \\ $l=1, \dots, L$ \end{tabular}  \vspace{1mm}\\

		$\ten{Z}^{(l)} = \sigma(\ten{F}^{(l)}) $ & \begin{tabular}[c]{@{}l@{}}Output of layers \\ $l = 1, \dots, L $,  where\\ $\sigma(\cdot)$ is a point-wise\\ activation function\end{tabular} \vspace{1mm} \\

		\hline
	\end{tabular}
\end{table}

Upon calculating the gradient of the cost function $L$ w.r.t. the input of the last layer of the network $\ten{Z}^{(L)}$, that is $\frac{\partial L  }{\partial \ten{F}^{(L)  }} \in \mathbb{R}^{I_{N+1}^{(L)} \times M}$, the error is then back-propagated through layers $l = L-1, L-2, \dots, 1$.
As a result the gradients of the cost function w.r.t. the factor matrices and the core tensor of the $l$-th Tucker tensor layer are computed as
\begin{equation}\label{eq:UG}
\begin{aligned}
&\frac{\partial  L }{\partial  \text{vec}(\mathbf{U}^{(n)}_{(l)} )   } = \bigg( \mathbf{P}_n \frac{\partial  \mathcal{W}_{(n)}^{(l)}    }   {\partial \mathbf{U}^{(n)}_{(l)}  } \bigg)^{T}  \frac{\partial  \ten{F}^{(l)}    }{\partial     \mathcal{W}_{(N+1)}^{(l)}  }^{T} \text{vec}(\ten{D}^{(l)})\\[5pt]
&\frac{\partial L  }{\partial \text{vec} (\mathcal{G}_{(N+1)}^{(l)})} = \frac{\partial \mathcal{W}_{(N+1)} ^{(l)} }{\partial  \mathcal{G}_{(N+1)  }  ^{(l)}   }^T \frac{\partial   \ten{F}^{(l)}   }{\partial \mathcal{W}_{(N+1)}^{(l)}  }^T\text{vec}(\ten{D}^{(l)})
\end{aligned}
\end{equation}
where
\begin{equation}\label{eq:D}
\hspace{-1.4mm }\ten{D}^{(l)} = 
\begin{cases}
& \hspace{-0.4cm} \frac{\partial  L }{\partial \ten{Z}^{(L)}    }  \frac{\partial  \ten{Z}^{(L)}  }  {\partial   \ten{F}^{(L)} } = \frac{\partial  L } {\partial \ten{F}^{(L)}   } , \hspace{3mm} \text{if $l=L$}\\\\
&\hspace{-0.4cm}{\scriptstyle \mathcal{W}_{(N+1)}^{(l+1)T} \ten{D}^{(l+1)} \odot \sigma '(\ten{F}^{(l)})}, \hspace{1mm}
\text{if $l = L-1, \dots, 1$}
\end{cases}
\end{equation}
and
\begin{equation}\label{eq:dFdWN1}
\frac{\partial  \ten{F}^{(l)} }{\partial   \mathcal{W}_{(N+1)}  ^{(l)} } = \ten{Z}^{(l-1)T} \otimes \mathbf{I}_{I_{N+1}^{(l)}}
\end{equation}
By substituting Equations (\ref{eq:dWdGN1}) and (\ref{eq:dFdWN1}) into Equation (\ref{eq:UG}),  we can now rewrite the gradients for the factor matrices $\mathbf{U}^{(n)}_{(l)}  $ and the core tensors $\tn{G}^{(l)}$ as
\begin{equation}\label{eq:gradients}
\begin{aligned}
&\frac{\partial  L }{\partial  \text{vec}(\mathbf{U}^{(n)}_{(l)} )   } = \bigg( \mathbf{P}_n \frac{\partial  \mathcal{W}_{(n)}^{(l)}    }   {\partial \mathbf{U}^{(n)}_{(l)}  } \bigg)^{T} \text{vec}\Big(\ten{D}^{(l)}  \ten{Z}^{(l-1)T}  \Big )\\
&\frac{\partial L  }{\partial  \mathcal{G}_{(N+1)}^{(l)}} = \mathbf{U}^{(N+1)T}_{(l)}  \ten{D}^{(l)}  \ten{Z}^{(l-1)T} \bigg( \Motimes_{ \substack{i=N }}^{1 } \mathbf{U}^{(i)}_{(l)}\bigg) 
\end{aligned}
\end{equation}
Moreover, if $n=N+1$, by making use of Equation (\ref{eq:dWndUn}),  we directly obtain the result
\begin{equation}\label{eq:dLdUN1}
\frac{\partial    L	}{\partial 		\mathbf{U}^{(N+1)}_{(l)}	} = \ten{D}^{(l)}\ten{Z}^{(l-1)T}\bigg( \Motimes_{ \substack{i=N }}^{1 } \mathbf{U}^{(i)}_{(l)}\bigg)\mathcal{G}_{(N+1)}^{(l)T}
\end{equation}
Similarly, $\frac{\partial  L		}{\partial  \mathbf{b}^{(l)}	} \in \mathbb{R}^{I_{N+1}^{(l)}}$ can be computed first by considering  $\frac{\partial  L		}{\partial  \mathbf{B}^{(l)}	} \in \mathbb{R}^{I_{N+1}^{(l)} \times M}$, 
\begin{equation}\label{eq:dLdB}
\begin{aligned}
\frac{\partial  L		}{\partial  	\text{vec}(\mathbf{B}^{(l)})	} &= \frac{\partial \ten{F}^{(l)}			}{\partial \mathbf{B}^{(l)}	} \text{vec}(\ten{D}^{(l)})\\
& = (\mathbf{I}_M \otimes \mathbf{I}_{I_{N+1}^{(l)}}) \text{vec}(\ten{D}^{(l)})
\end{aligned}
\end{equation}
Hence, from the rules of the Kronecker product,  $\frac{\partial  L  }{\partial   \mathbf{B}^{(l) }    } = \ten{D}^{(l)}$, which implies that, 
\begin{equation}\label{eq:dLdb}
\begin{aligned}
\frac{\partial  L }{\partial  \mathbf{b}^{(l)}  } &= \bigg(   \frac{\partial  \mathbf{B}^{(l)}}{\partial \mathbf{b}^{(l)}  }   \bigg)^T \text{vec}(\ten{D}^{(l)} )\\
& = \bigg(\mathbf{1}_M^T \otimes \mathbf{I}_{I_{N+1}^{(l)}} \bigg )  \text{vec}(\ten{D}^{(l)} )
\end{aligned}
\end{equation}
The forward and back-propagation procedures are summarized in Algorithm \ref{algo:tkdnn}.
\begin{algorithm}[t]
	\caption{Forward and back-propagation for TTLs}
	\label{algo:tkdnn}
	\begin{algorithmic}[1]
		\State\textbf{Input:} Dataset $\ten{X} \in \mathbb{R}^{\underline{I}_N \times M}$
		\State
		\State $\ten{Z}^{(0)} = \ten{X}$, and initialize all $\mathbf{U}^{(n)}_{(l)}$ and $\mathcal{G}_{(N+1)}^{(l)}$
		\State 
		\State \textbf{Forward propagation:}
		\For {$l = 1, \dots, L$}
		\State Compute $\mathcal{W}_{(N+1)}^{(l)}$ via Equation (\ref{eq:w_unfolded})
		\State $\ten{F}^{(l)} = \mathcal{W}_{(N+1)}^{(l)} \ten{Z}^{(l-1)} + \mathbf{B}^{(l)}$
		\State $\ten{Z}^{(l)} = \sigma(\ten{F}^{(l)}) $
		\State Store $\ten{F}^{(l)}$, $\ten{Z}^{(l)}$
		\State Compute and store $\sigma ' (\ten{F}^{(l)})$
		\EndFor
		\State
		\State \textbf{Back-propagation:}
		\For {$l = L, L-1 \dots, 1$}
		\State Compute $\ten{D}^{(l)}$ via Equation (\ref{eq:D})
		\State Compute the gradients via Equations (\ref{eq:gradients})-(\ref{eq:dLdb}) 
		\State $\mathbf{U}^{(n)}_{(l)} \leftarrow \mathbf{U}^{(n)}_{(l)} - \eta \bigg( \frac{\partial  L }{\partial  \mathbf{U}^{(n)}_{(l)}    } \bigg)$
		\State $\mathcal{G}_{(N+1)}^{(l)}  \leftarrow \mathcal{G}_{(N+1)}^{(l)} - \eta \bigg( \frac{\partial  L }{\partial  \mathcal{G}_{(N+1)}^{(l)}    } \bigg)$
		\State $\mathbf{b}^{(l)} \leftarrow \mathbf{b}^{(l)} - \eta \bigg(    \frac{\partial 	L	}{\partial \mathbf{b}^{(l)}  	}       \bigg) $ 
		\EndFor
	\end{algorithmic}
\end{algorithm}

\section{Simulation Results}

A comprehensive experimental validation of the proposed concept includes the discussion of physical interpretability, followed by intuitive simulation results on synthetic data as well as standard benchmark datasets. 
\vspace{-2mm}
\subsection{Physical Interpretability}\label{sec:int}

Based on the general equation for the TTL model in (\ref{eq:F}), it is natural to ask whether the matrices  $\mathbf{U}^{(n)} \in \mathbb{R}^{I_n \times R_n}$ are physically related to the data tensors $\tn{X} \in \mathbb{R}^{I_1 \times I_1 \times \cdots \times I_N}$.  To this end, consider the identity, 
\begin{equation}\label{eq:equivalence}
\begin{aligned}
F(\tn{X}) &= \bigg [ \mathbf{U}^{(N+1)}\mathcal{G}_{(N+1)} \Motimes_{i=N}^{1} \mathbf{U}^{(i)T} \bigg] \text{vec}(\tn{X}) + \mathbf{b} \\[2pt]
&=\mathbf{U}^{(N+1)} \mathcal{G}_{(N+1)} \text{vec}\big(\tn{X} \times_1 \mathbf{U}^{(1)T} \times_2 \\ 
&\hspace{1cm}\times_2\mathbf{U}^{(2)T} \times \cdots \times_N \mathbf{U}^{(N)T} \big) + \mathbf{b}
\end{aligned}
\end{equation}

This implies that the $n$ physical modes, of dimensionality $I_n$, of tensors $\tn{X}$ are directly intertwined with the corresponding factor matrices $\mathbf{U}^{(n)}$. As a direct consequence, certain gradient terms $\frac{\partial L}{\partial \mathbf{U}^{(n)}}$ will carry more significant information than others, depending on which modes $I_n$ of the original data tensors are richer in structure. This is of great practical importance, as one of the most well-known problems with NNs is their black-box nature, while the proposed method offers  viable means to aid the understanding of which data features have a higher influence on their training process.

Figure \ref{fig:equivalence} offers a graphical representation of Equation (\ref{eq:equivalence}). Each node represents a tensor, the order of which is determined by the number of edges it connects to. The edges represent the modes, and their labels the corresponding dimensionality. 
\begin{remark}
	From Figure \ref{fig:equivalence}, it becomes apparent that each factor matrix, $\mathbf{U}^{(n)}$, and hence each gradient, $\frac{\partial  L}{\partial  \mathbf{U}^{(n)}  }$, is associated with the respective mode-$n$, suggesting that the gradients themselves will carry valuable information related to the underlying data structure.
\end{remark}

\begin{remark}
	A more rigorous mathematical argument for the above can be provided via inspection of (\ref{eq:dLdU}) and (\ref{eq:dWndUn}), which shows that $\frac{\partial  L}{\partial  \mathbf{U}^{(n)}  }$ is dependent on $\mathcal{G}_{(n)}$, and hence $\mathcal{W}_{(n)}$. Also, the unfolding of the weight tensor along mode-$n$ directly reveals the connection with mode-$n$ of the input data $\tn{X}$, and it is therefore natural that each $\frac{\partial  L}{\partial  \mathbf{U}^{(n)}  }$ has modal information embedded within. This conclusion could not have been reached without a solid underlying mathematical analysis.
\end{remark}
%
%
\begin{figure}[t]
	\centering
	\includegraphics[width=0.8\linewidth,  trim={0cm 0cm 0cm 0cm}, clip]{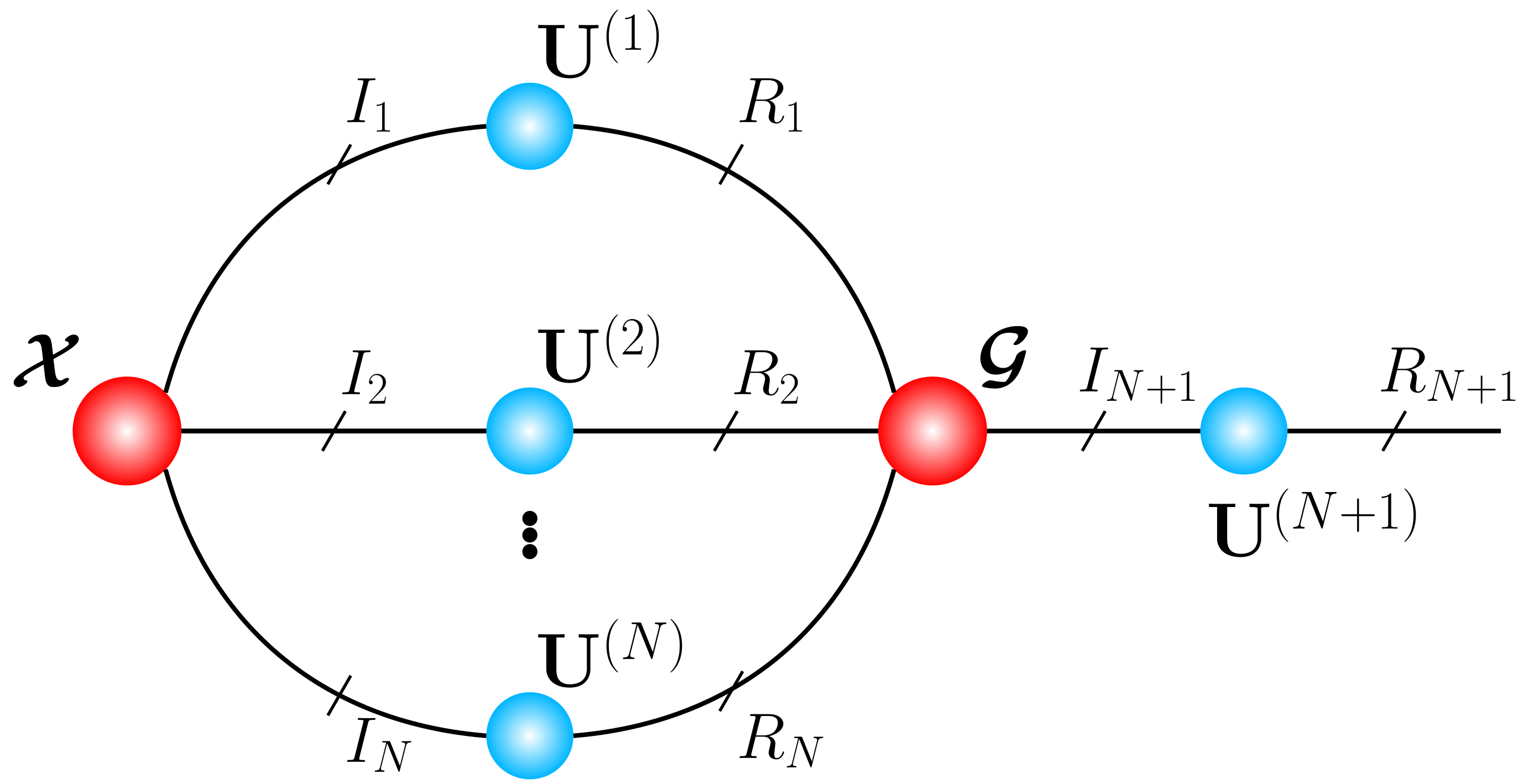}
	\caption{Graphical representation of Equation (\ref{eq:equivalence}). \vspace{-5mm}}
	\label{fig:equivalence}
\end{figure}
\vspace{-5mm}

\subsection{Synthetic Data}
To provide more concrete intuition into the above discussion, we evaluated our model on two synthetic datasets, referred to as Synthetic Dataset 1 (SD1), and Synthetic Dataset 2 (SD2). Each dataset was composed of $28\times 28$ grey-scale images ($2$-nd order tensors), which were white everywhere except for randomly selected rows for SD1 and randomly selected columns for SD2, designated black. Examples of images from these datasets are shown in Figure \ref{fig:synth}.

\begin{figure}[b]
	\centering
	\includegraphics[width=0.6\linewidth,  trim={0 3cm 0 4.5cm}, clip]{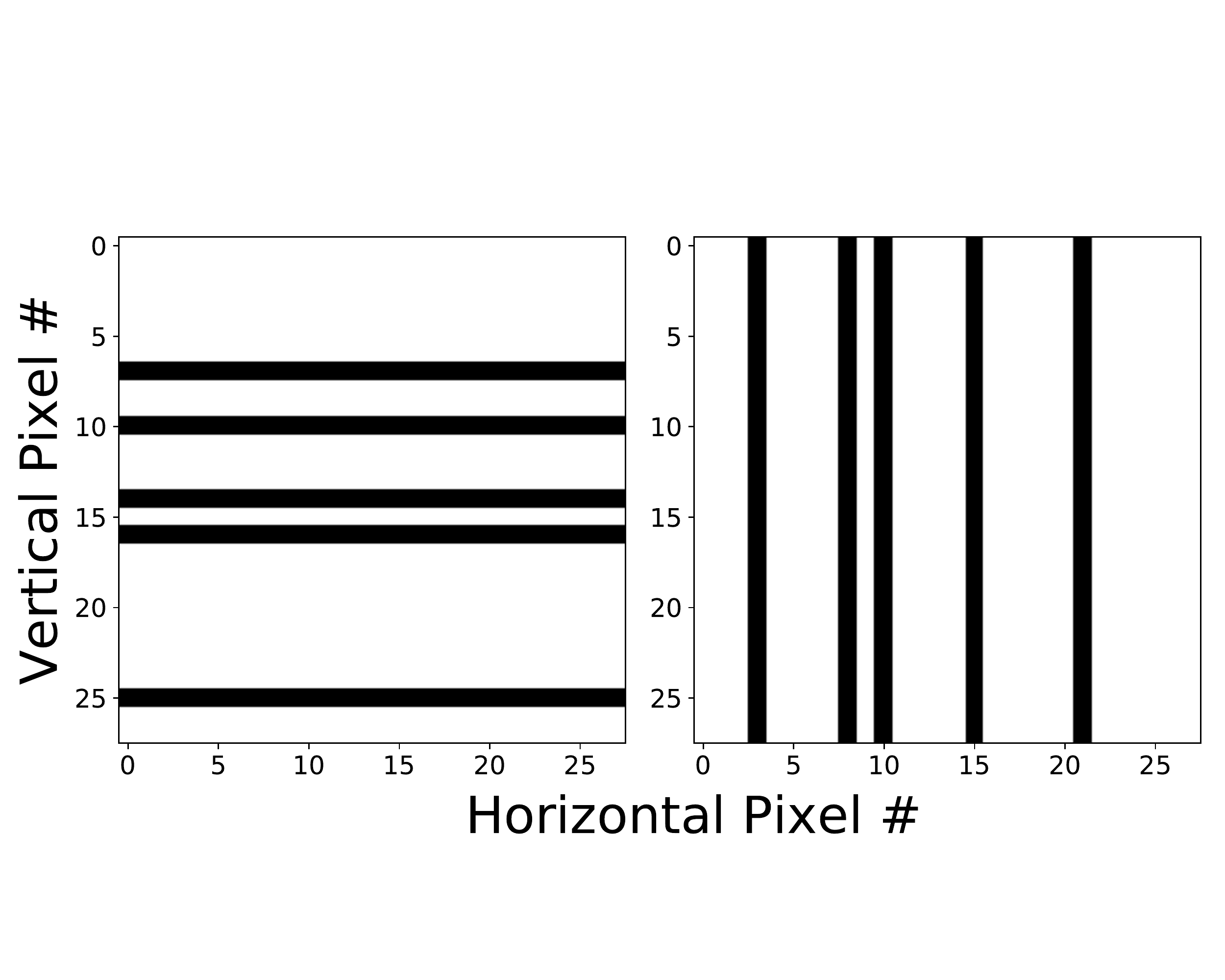}
	\caption{Sample images from two synthetic datasets. \textit{Left}: Synthetic Dataset 1. \textit{Right}: Synthetic Dataset 2.}
	\label{fig:synth}
\end{figure}

\begin{figure*}[t]
	\centering
	\includegraphics[width=1\linewidth,  trim={0 0cm 0 0cm}, clip]{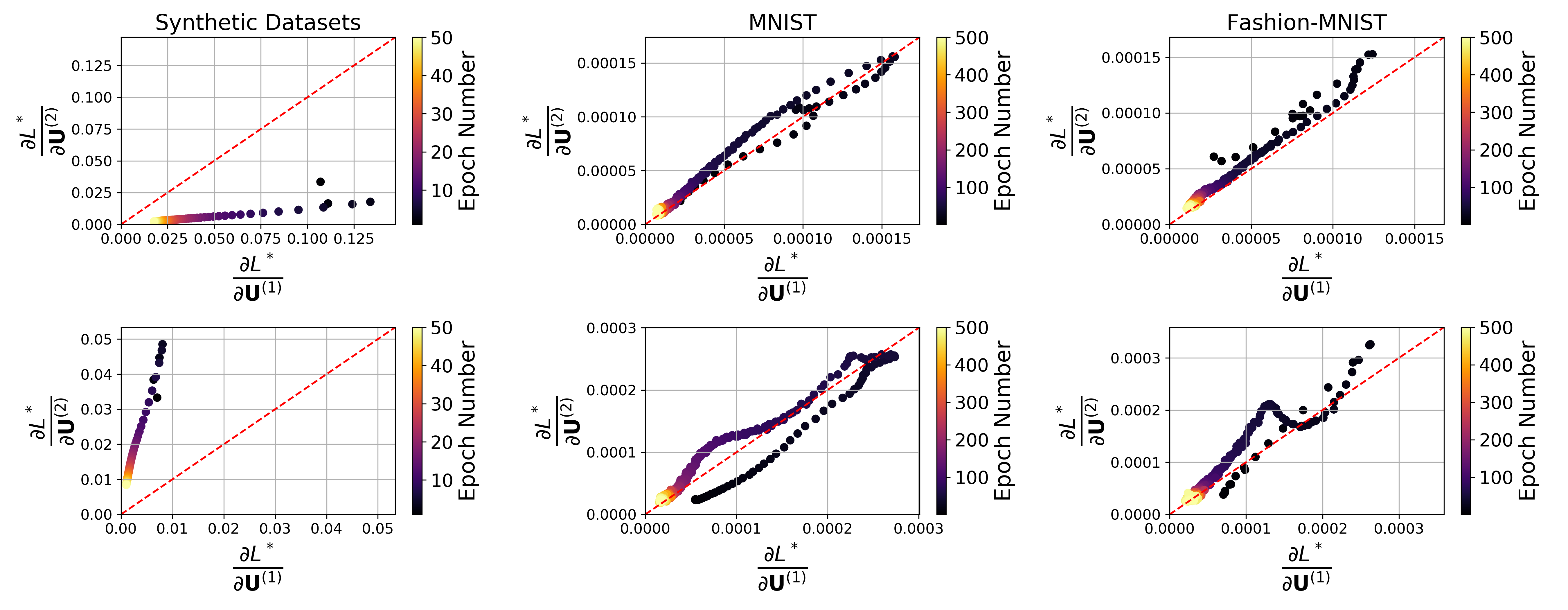}
	\caption{Gradients across epochs -- TTL applied to various datasets. \textit{Left column. Top row:}  SD1; \textit{bottom row:} SD2. The disparity in modal structural information is reflected in the magnitude of the normalized gradients. \textit{Middle column. Top row:} MNIST and TTL with $\tn{G} \in \mathbb{R}^{10 \times 10 \times 30}$; \textit{bottom row:} MNIST and TTL with $\tn{G} \in \mathbb{R}^{5 \times 5 \times 10}$. No difference in modal structural information is observed. \textit{Right column. Top row:}  Fashion-MNIST and TTL with  $\tn{G} \in \mathbb{R}^{10\times 10 \times 30}$; \textit{bottom row:}  Fashion-MNIST and TTL with  $\tn{G} \in \mathbb{R}^{5\times 5 \times 10}$. Results suggest that the second mode carries more structural information.}
	\label{fig:sims}
\end{figure*}

Such a structure was chosen for SD1 and SD2 to ensure that one tensor mode is far richer in structure than the other. It is therefore expected that $\frac{\partial  L}{\partial  \mathbf{U}^{(1)}}$ will be more important than $\frac{\partial  L}{\partial  \mathbf{U}^{(2)}}$ for SD1, and vice-versa for SD2. For quantitative assessment, at each epoch we stored the normalized Frobenius norm of each gradient, defined as 

\begin{equation}
	\frac{\partial L^*}{\partial  \mathbf{U}^{(n)}} = \bigg|\bigg| \frac{\partial  L}{\partial  \mathbf{U}^{(n)}} \bigg|\bigg|_F\bigg/(I_n R_n)
\end{equation}

%

Simulations were run for 50 epochs and the results are shown in the left-most column of Figure \ref{fig:sims}, which suggest that the proposed model is capable of capturing well the relative importance of the modal structure within data. In the case of SD1, mode-1 is richer in structure, and this is indicated by the fact that $\frac{\partial L^*}{\partial  \mathbf{U}^{(1)}} > \frac{\partial L^*}{\partial  \mathbf{U}^{(2)}}$ for all epochs. The converse is true for SD2. This demonstrates that the TTL implicitly resolves the notorious black-box nature of NNs, by offering physically meaningful information inferred from the gradients of the trained weights.

%

\subsection{MNIST Dataset}\label{sec:mnist}
We now demonstrate the desirable compression properties of the TTL by applying it to the MNIST dataset \cite{Lecun1998} for the task of handwritten digit recognition.  The MNIST dataset is composed of 60000 $28 \times 28$ greyscale images for training and 10000 for testing.  We used a neural network with 1 hidden layer and a rectified linear unit (ReLU) activation function. The hidden layer was of 300 neurons in size, and in our investigation it was replaced by the proposed TTL. The weight tensor, $\tn{W} \in \mathbb{R}^{28\times 28 \times 300}$, is in uncompressed format if its core $\tn{G} \in \mathbb{R}^{28 \times 28 \times 300}$. The compression factor (CF) was computed as the ratio of the number of elements of a ``full" core tensor and factor matrices to the number of elements in a compressed core and its respective matrices. Table \ref{table:MNIST} shows results for different CFs, over a training period of 500 epochs. 

The TTL achieved a compression factor (CF) of 18.73 with accuracy of $95.6\%$, and CF of 66.63 with accuracy of $93.3\%$, corresponding to a decrease of only slightly over $2\%$ from the original uncompressed network. Figure \ref{fig:sims} shows scatter plots of the normalized gradients, $\frac{\partial  L^*}{\partial  \mathbf{U}^{(1)}}$ and $\frac{\partial  L^*}{\partial  \mathbf{U}^{(2)}}$, for both the compression factors in Table \ref{table:MNIST}. Observe that, for MNIST, both modes have comparable structural significance, as suggested by a simple visual perspective on the images (white numbers on black background). 

\subsection{Fashion-MNIST Dataset}\label{sec:fmnist}

\begin{table}[b]
	\centering
	\renewcommand{\arraystretch}{1.5}
	\caption{TTL performance for various Compression Factors (CF).}
	\label{table:MNIST}
	\begin{tabular}{cccc}
		\hline
		\textbf{Dataset}                                                                                       & \textbf{CF}              & \textbf{Core Size}                            & \textbf{Accuracy (\%)} \\ \hline
		\multicolumn{1}{c|}{\multirow{3}{*}{{MNIST}}}                                                   & \multicolumn{1}{c|}{1}   & \multicolumn{1}{c|}{$28\times 28 \times 300$} & $95.9\%$               \\ \cline{2-4} 
		\multicolumn{1}{c|}{}                                                                                  & \multicolumn{1}{c|}{18.73}  & \multicolumn{1}{c|}{$10 \times 10 \times 30$} & $95.6\%$               \\ \cline{2-4} 
		\multicolumn{1}{c|}{}                                                                                  & \multicolumn{1}{c|}{66.63} & \multicolumn{1}{c|}{$5 \times 5 \times 10$}   & $93.3\%$               \\ \hline
		\multicolumn{1}{c|}{\multirow{3}{*}{{\begin{tabular}[c]{@{}c@{}}Fashion\\ MNIST\end{tabular}}}} & \multicolumn{1}{c|}{1}   & \multicolumn{1}{c|}{$28\times 28 \times 300$} & $86.3\%$               \\ \cline{2-4} 
		\multicolumn{1}{c|}{}                                                                                  & \multicolumn{1}{c|}{18.73}  & \multicolumn{1}{c|}{$10\times 10 \times 30$}  & $85.4\%$               \\ \cline{2-4} 
		\multicolumn{1}{c|}{}                                                                                  & \multicolumn{1}{c|}{66.63} & \multicolumn{1}{c|}{$5\times 5 \times 10$}    & $82.3\%$             
	\end{tabular}
	\vspace{-3mm}
\end{table}

\begin{figure*}[t!]
	\centering
	\includegraphics[width=1\linewidth,  trim={0 0cm 0 0cm}, clip]{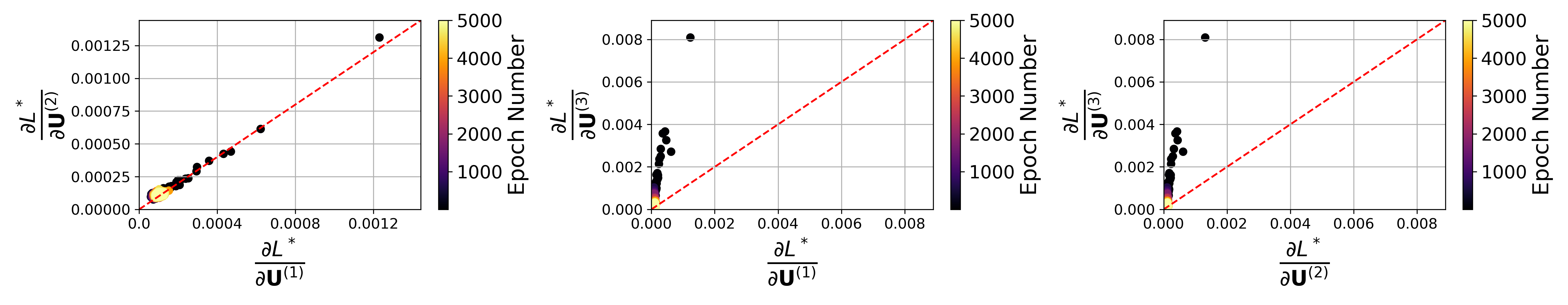}
	\caption{Gradients across epochs -- TTL applied to CIFAR-10. A TTL with a core tensor $\tn{G} \in \mathbb{R}^{10 \times 10 \times 3 \times 10}$ is employed, to unveil more structural information within mode-3 than within the other modes, thus indicating that the color scheme of the images within CIFAR-10 plays a major role in classification tasks.}
	\label{fig:cifar}
\end{figure*}

We further evaluated the proposed TTL model on the Fashion-MNIST dataset \cite{Xiao2017}, which is intended to serve as a direct, more complex replacement of the original MNIST. Fashion-MNIST consists of $28\times 28$ greyscale images depicting fashion products belonging to a total of 10 different classes. The training and testing sets consist of 60000 and 10000 images, respectively. We employed a neural network with 2 hidden layers of sizes 300 ($l=1$) and 200 ($l=2$), and ReLU activation functions, replaced the layer $l=1$ with our TTL, and the layer was compressed by the same degree as for MNIST. The uncompressed network attained an accuracy of $86.3\%$, while, at a loss of less than $4\%$ accuracy, the TTL achieved a CF of 66.63, obtaining a $82.3\%$ performance.

Differently from MNIST, Fashion-MNIST entails a disparity in modal structural information. For example, a coat and a dress occupy roughly the same vertical space on the images (we refer to \cite{Xiao2017}), while the difference between the two products is revealed horizontally -- mode-$1$ and mode-$2$. This disparity implies that mode-$2$ should play a relatively more important role  than mode-$1$ in training. This is reflected in Figure \ref{fig:sims}, where it is shown that, on average, $\frac{\partial L^*}{\partial  \mathbf{U}^{(1)}} > \frac{\partial L^*}{\partial  \mathbf{U}^{(2)}}$. This further supports our arguments in Section \ref{sec:int}, that, within the proposed TTL, each gradient $\frac{\partial  L }{\partial    \mathbf{U}^{(n)}  }$ is directly intertwined with its corresponding mode-$n$. Some gradients will be more informative about the training process, depending on which data mode carries more structure. This is made explicit owing to the analytical derivation of back-propagation, in terms of the TKD, which allowed us to leverage on, and gain insight from the expressive power of tensors.

\subsection{CIFAR-10 Dataset}

\begin{figure}[t]
	\centering
	\includegraphics[width=0.7\linewidth,  trim={5cm 16.8cm 8.5cm 4cm}, clip]{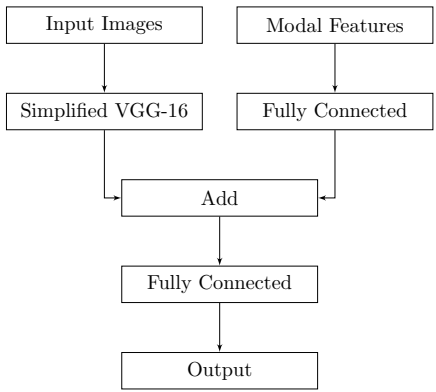}
	\caption{Modal-VGG architecture. Modal features are passed to the network in vector format as expressed in (\ref{eq:mode}), while the simplified VGG-16 is the normal VGG-16 without the last two convolutional layers.}
	\label{fig:structure}
\end{figure}

The same network employed in Sections \ref{sec:mnist}-\ref{sec:fmnist} was applied on the CIFAR-10 dataset \cite{Krizhevsky2009}, with the TTL substituting the first hidden layer comprised of 300 nodes. Since CIFAR-10 is composed of $32 \times 32 \times 3 $ RGB images (50000 for training), the decomposed uncompressed weight tensor corresponds to a core of size $\tn{G} \in \mathbb{R}^{32 \times 32 \times 3 \times 300}$. We compressed the layer by a factor of 152.45 by employing a core of size $\tn{G} \in \mathbb{R}^{10 \times 10 \times 3 \times 10}$. Simulations were ran for both the uncompressed and compressed networks, achieving accuracy rates of 53.8\% and 48.03\% respectively. Such a drop in performance compared to the MNIST and the Fashion-MNIST datasets was expected as the CIFAR-10 dataset is naturally more complex. However, by virtue of the proposed framework, it was possible to analyze the relative importance of the modal structure within the NN. Indeed, Fig. \ref{fig:cifar} shows that mode-3, corresponding to RGB information, is more significant than the other two for classification purposes; in turn, it can be used as an additional input feature to more complex networks.

This was achieved by combining the well-known VGG-16 convolutional network \cite{Simonyan2014} with a simple network consisting of a single fully connected layer, of 512 units, which receives as input  modal information for each sample image, as illustrated in Fig. \ref{fig:structure}. We refer to this model as Modal-VGG. Specifically, for each image $\tn{X} \in \mathbb{R}^{32 \times 32 \times 3}$ we fed the vectorized $3$-rd mode covariance, that is
\begin{equation}\label{eq:mode}
	\mathbf{x} = \text{vec}(\mathcal{X}_{(3)}^T \mathcal{X}_{(3)}) \in \mathbb{R}^9
\end{equation}
To highlight the significance of the $3$-rd mode, the VGG-16 was first simplified by removing its last two convolutional layers, then trained on CIFAR-10 in two ways: (i) alone and (ii) in combination with another network which analyzes modal information. The outputs of the first fully connected layer of the simplified VGG-16 were added to those of the second network, before being fed through a final fully connected layer. 

\begin{remark}
	The Modal-VGG is an example of a straightforward methodology which is very useful in practice. Namely, if simple networks fail, instead of discarding completely their results, the TTL offers opportunities to retain some information of their training, which can then be fed to more complex networks. This, in turn, allows large NNs to be simplified and reduce their computational costs thanks to the so performed data augmentation.  
\end{remark}

Simulations were ran on an NVIDIA GeForce RTX 2080 Ti GPU. TABLE \ref{table:CIFAR} shows that, for case (i) above, training speed was increased by over 12\%, but the model accuracy decreased from 91.3\% (for the full VGG-16) to 86.2\%. The benefits of the proposed framework are evident in case (ii), where augmenting the data with the relevant modal features allowed for the accuracy to increase to 89.2\%, while still maintaining about a 10\% reduction in overall training time when compared to the full VGG-16 model. For rigour, the same experiment was carried out on the $1$-st and $2$-nd mode vectorized covariances, for which the Modal-VGG in Fig. \ref{fig:structure} attained accuracies of 86.4\% and 86.9\%, respectively. This further highlights the importance of mode-$3$ for the CIFAR-10 dataset and demonstrates that the TTL can indeed capture modal information, which, in turn, may be exploited to enhance physical meaning in training and reduce training time while at the same time maintaining comparable performance.

\begin{table}[]
	\centering
	\renewcommand{\arraystretch}{1.3}
	\caption{Modal-VGG performance.}
	\label{table:CIFAR}
	\begin{tabular}{ccc}
		\hline
		\textbf{Network}                                                                                        & \textbf{Accuracy (\%)}      & \textbf{\begin{tabular}[c]{@{}c@{}}Training Time\\ Speed-up (\%)\end{tabular}} \\ \hline
		\multicolumn{1}{c|}{{\begin{tabular}[c]{@{}c@{}}Original VGG-16;\\ {} \end{tabular}}}                                                           & \multicolumn{1}{c|}{$91.3$} & --                                                                               \\ \hline
		\multicolumn{1}{c|}{{\begin{tabular}[c]{@{}c@{}}Modal-VGG;\\ no modal information\end{tabular}}} & \multicolumn{1}{c|}{$86.2$} & $\sim 12$                                                                        \\ \hline
		\multicolumn{1}{c|}{{\begin{tabular}[c]{@{}c@{}}Modal-VGG;\\ mode-1\end{tabular}}}               & \multicolumn{1}{c|}{$86.4$} & $\sim 10$                                                                        \\ \hline
		\multicolumn{1}{c|}{{\begin{tabular}[c]{@{}c@{}}Modal-VGG;\\ mode-2\end{tabular}}}               & \multicolumn{1}{c|}{$86.9$} & $\sim 10$                                                                        \\ \hline
		\multicolumn{1}{c|}{{\begin{tabular}[c]{@{}c@{}}Modal-VGG;\\ mode-3\end{tabular}}}               & \multicolumn{1}{c|}{$89.2$} & $\sim 10$                                                                       
	\end{tabular}
\vspace{-2mm}
\end{table}

%
%

%
%

\vspace{-2mm}
\section{Conclusions}

We have introduced the Tucker Tensor Layer (TTL) as an alternative to the dense weight matrices of fully connected layers in Deep  Neural Networks (DNNs). This has been achieved by treating the weight-matrix as an unfolding of a higher-order tensor, followed by its decomposition and compression via the Tucker decomposition (TKD). By extending the notion of matrix derivatives to tensors, we have derived the tensor-valued back-propagation within our proposed framework and have demonstrated physical interpretability of the proposed model. Experiments on synthetic data, Fashion-MNIST and CIFAR-10, show that, owing to the expressive power of tensors and in particular the TKD, the proposed framework has enabled us to gain physical insights into the network training process, by providing explicit information on the relative importance of each data mode in training. The exploitation of the relative importance of data modes has enabled us to train a data-augmented network comprising a simplified version of the VGG-16, which resulted in a 10\% decrease in training time, at a small cost in terms of final accuracy. Moreover, the compression benefits of our model have been demonstrated on both the MNIST and Fashion-MNIST datasets, with the TTL achieving a 66.63 fold compression while maintaining comparable performance to the uncompressed NN. This further illustrates the ability of the proposed TTL to help resolve the two main bottlenecks of NNs, namely their black-box nature and more efficient training of its millions of parameters.


\ifCLASSOPTIONcaptionsoff
\newpage
\fi



%
\bibliographystyle{IEEEtran}
\vspace{30mm}
\bibliography{references}

%




\end{document}